\documentclass[american,a4paper,runningheads,envcountsect]{llncs}
\pdfoutput=1
\usepackage{babel}
\usepackage[utf8]{inputenc}
\usepackage[T1]{fontenc}

\usepackage{amsmath}
\usepackage{amsthm}
\usepackage{amssymb}
\usepackage{mathtools}
\usepackage{mathrsfs}

\usepackage{dsfont} 

\usepackage{graphicx}
\usepackage{float} 

\usepackage{subcaption}
\usepackage[ruled,vlined,linesnumbered]{algorithm2e}

\usepackage{tikz}
\usetikzlibrary{shapes,arrows,positioning,calc} 
\usetikzlibrary{matrix}

\usepackage{csquotes}
\usepackage{booktabs}
\usepackage{enumitem}

\usepackage[shrink=30,babel,kerning=true]{microtype}

\usepackage[subtle]{savetrees}

\usepackage{fca_mod}
 
\usepackage{todonotes}
\presetkeys{todonotes}{color=blue!5}{}

\usepackage[backend=biber,style=numeric-comp,doi=false,isbn=false,%
url=false]{biblatex}
\usepackage[colorlinks,citecolor=blue,linkcolor=blue,urlcolor=black,linktocpage]{hyperref}
\hypersetup{
  bookmarks=true,
  bookmarksnumbered=true,
  pdftitle={Triadic Exploration and Exploration with Multiple Experts},
  pdfauthor={\textcopyright Maximilian Felde and Gerd Stumme},
  pdfsubject={ Formal Concept Analysis (FCA) provides a method called \emph{attribute exploration} which helps a domain expert discover structural dependencies in knowledge domains that can be represented by a formal context (a cross table of objects and attributes).  Triadic Concept Analysis is an extension of FCA that incorporates the notion of conditions.  Many extensions and variants of attribute exploration have been studied but only few attempts at incorporating multiple experts have been made.  In this paper we present \emph{triadic exploration} based on Triadic Concept Analysis to explore \emph{conditional attribute implications} in a triadic domain.  We then adapt this approach to formulate attribute exploration with multiple experts that have different views on a domain.},
  pdfkeywords={Formal~Concept~Analysis; Triadic~Concept~Analysis; Attribute~Exploration},
  }

\usepackage[all]{hypcap}

\usepackage{cleveref}
\let\cref\Cref
\crefname{strategy}{Strategy}{Strategies}
\crefname{defn}{Definition}{Definitions}
\Crefname{defn}{Definition}{Definitions}

\newtheorem{thm}{Theorem}[section]

\newtheorem{prop}[thm]{Proposition}

\theoremstyle{definition}

\theoremstyle{remark}
\newtheorem{exmpl}[thm]{Example}

\renewcommand{\GMI}[1][]{(G,M,\relI_{#1})}
\newcommand{\GMII}[1][]{(G_{#1},M,\relI_{#1})}

\renewcommand{\implies}{\rightarrow}

\newcommand{\cimplies}[1]{\xrightarrow[]{#1}}

\newcommand{\GMBY}[1][]{\left(G,M,B,Y^{#1}\right)}
\newcommand{\K}{\context}
\renewcommand{\L}{\mathcal{L}}
\newcommand{\T}{\context[T]}
\newcommand{\E}{\context[E]}
\newcommand{\C}{\context[C]}

\newcommand{\Impm}{\ensuremath{\operatorname{Imp}_M}}

\newcommand{\Sat}[1][\context]{\operatorname{Sat}(#1)}

\renewcommand{\epsilon}{\varepsilon}
\renewcommand{\phi}{\varphi}

\newcommand{\Gex}{G_{\text{ex.}}}
\newcommand{\Mex}{M_{\text{ex.}}}
\newcommand{\Bex}{B_{\text{ex.}}}
\newcommand{\Tex}{\T_{\text{ex.}}}
\newcommand{\Mo}{\text{Mo-Fr}}
\renewcommand{\Sat}{\text{Sat}}
\newcommand{\Sun}{\text{Sun}}
 
\begin{document}

\title{Triadic Exploration and Exploration with Multiple Experts}
\subtitle{}

\author{Maximilian Felde\inst{1,2} \and Gerd Stumme\inst{1,2}}

\date{\today} 
 
\institute{%
  Knowledge \& Data Engineering Group,
  University of Kassel, Germany\\[0.5ex]
  \and
  Interdisciplinary Research Center for Information System Design\\
  University of Kassel, Germany\\[0.5ex]
  \email{felde@cs.uni-kassel.de, stumme@cs.uni-kassel.de}
}
\maketitle

\begin{abstract}

  Formal Concept Analysis (FCA) provides a method called \emph{attribute exploration} which helps a domain expert discover structural dependencies in knowledge domains that can be represented by a formal context (a cross table of objects and attributes).
  Triadic Concept Analysis is an extension of FCA that incorporates the notion of conditions.
  Many extensions and variants of attribute exploration have been studied but only few attempts at incorporating multiple experts have been made.
  In this paper we present \emph{triadic exploration} based on Triadic Concept Analysis
  to explore \emph{conditional attribute implications} in a triadic domain.
  We then adapt this approach to
   formulate attribute exploration with multiple experts that have different views on a domain.
 
\end{abstract}

 
\section{Introduction}
\label{sec:introduction}
Attribute exploration~\cite{ganter1984two} is a well established knowledge acquisition method from the field of Formal Concept Analysis (FCA)~\cite{GanterWille1999}.
Attribute exploration works on domains that can be represented as binary tabular data of objects and attributes (also called features or properties).
It helps a domain expert to uncover the dependency structure of attributes  of the domain.
  For non-binary tabular data the method of \emph{conceptual scaling}, cf. \cite{ganter89}, can be used to transform non-binary attributes into binary ones.

Attribute exploration is based on the idea that we extend domain information through a domain expert.  To this end, attribute exploration uses a question-answer scheme to extract dependency information about attributes.
The questions are in the form of \emph{implications}, for example, \emph{do attributes A and B imply attribute C?} (also written as $AB\implies C?$).
The expert's task is to confirm or refute the validity of such implications in the domain.
If the expert refutes the validity of an implication she has to offer a counterexample, for example, in case of the question $AB\implies C?$ an object of the domain that has the attributes $A$ and $B$ but lacks attribute $C$.

The attribute exploration algorithm asks these questions in an optimized manner such that the expert has to answer as few questions as possible until the validity of every conceivable implication can be inferred from the answers given by the expert.  This is the case when every implication either follows from the set of implications accepted as valid or is contradicted by one of the examples given by the expert. 

The basic version of attribute exploration requires an all-knowing expert of the domain, i.e. an expert who can answer any question about the domain correctly.
It was introduced by Ganter in \cite{ganter1984two}.
Since then, many variants and extensions of attribute exploration have been studied.
A good overview can be found in the book \emph{Conceptual Exploration} by Ganter and Obiedkov~\cite{ganter2016conceptual}.
These extensions and variants notably include: Attribute exploration with background knowledge and exceptions~\cite{ganter1999attribute,stumme96attribute}, where the idea is to support the exploration with prior knowledge about some of the relations between attributes, for example if one attribute is the negation of another; attribute exploration with partial information~\cite{holzer2001dissertation,holzer2004knowledgeP1,holzer2004knowledgeP2}, where the expert is not required to be all-knowing and is also allowed to answer \emph{I do not know} in addition to confirming or refuting a question. Further, the expert is not required to fully specify a counterexample as long as the specified parts contradict the implication in question; and a sketch of how to explore triadic formal contexts~\cite{ganter2016conceptual,GanterObiedkov04Triadic}, where the idea of attribute exploration is transferred to triadic concept analysis (an extension of FCA with conditions~\cite{lehmann1995triadic}). We elaborate further on this in \cref{sec:triadic-exploration}.

However, most of the extensions and variants of attribute exploration that have been studied are based on the idea of a single expert answering the questions.
As far as we know, there exist only a few papers that mention exploration with multiple experts, notably: Paper \cite{conf/iccs/Kriegel16} deals with how to perform exploration in parallel and potentially offers a way to speed up the exploration with multiple experts; \cite{conf/iccs/HanikaZ18} addresses collaborative conceptual exploration based on the notions of local experts for subdomains of a given knowledge domain; and \cite{DBLP:journals/corr/abs-1908-08740} studies attribute exploration in a collaborative exploration setting with multiple experts who share the same view on the domain but only have partial knowledge thereof.

When we explore a domain with multiple experts, one of the fundamental problems we face is that different views on a domain, for example
different opinions whether an object has an attribute or not, or whether an implication is valid or not in a domain
, are impossible to resolve by combining different pieces of information into one.
Either, because there is no clear \emph{right} or \emph{wrong}, e.g. in case of opinions, or simply because we can not know which information to trust most.
And, even if we used methods such as majority-voting on information, there is a reasonable chance that the result is not always correct.
Combined with the inherent non-robustness of implication theories, i.e., small changes in the underlying data can lead to a very different theory, this suggests that merging different views on a domain is a bad idea for attribute exploration.
If we take a closer look at the publications mentioned before, we see that all three avoid this issue in their own way.
In \cite{conf/iccs/Kriegel16} the experts all have the same complete knowledge about the domain; in \cite{conf/iccs/HanikaZ18} the local experts have partial knowledge about the same consistent domain knowledge; and, in \cite{DBLP:journals/corr/abs-1908-08740} the problem was also avoided by defining expert knowledge as partial knowledge of some consistent domain knowledge.

Attribute exploration where multiple experts can have truly different and even opposing views on the domain has to the best of our knowledge not yet been studied.
To this end we develop \emph{triadic exploration} based on ideas presented by Ganter and Obiedkov in \cite{GanterObiedkov04Triadic}.
We then adapt triadic exploration to the setting of multiple experts with different views on a domain and thus provide a step in the direction of attribute exploration with multiple experts.

The paper is structured as follows: We begin by giving a brief introduction to the problem
in \cref{sec:introduction}.
We recollect some fundamentals of Formal and Triadic Concept Analysis in \cref{sec:dyadic-triadic-contexts}, in particular \emph{formal} and \emph{triadic contexts}, \emph{attribute implications}, the \emph{relative canonical base} and \emph{attribute exploration}.
In \cref{sec:triadic-exploration}, we discuss implications in the triadic setting, in particular, we focus on \emph{conditional attribute implications}.
Subsequently, we formulate \emph{triadic exploration}.
In \cref{sec:appl-mult-experts}, we discuss how to adapt \emph{triadic exploration} to model
attribute exploration with multiple experts with different views.
Finally, \cref{sec:conclusion-outlook} contains conclusion and outlook.  
Note that for this paper we do not provide a separate section for related work, instead we address related work throughout the paper whenever appropriate.

\section{Dyadic and Triadic Formal Contexts}
\label{sec:dyadic-triadic-contexts}
In this section we recollect the fundamentals of (dyadic) Formal Concept Analysis and Triadic Formal Concept Analysis (TCA).
We mostly rely on \cite{Wille82,GanterWille1999} for FCA and on \cite{lehmann1995triadic,wille1995basic} for TCA.
We begin with the definition of \emph{formal contexts} and associated notions.  We then introduce \emph{triadic contexts} and give an example which will serve as our running example for the remainder of this paper.  Afterwards, we briefly cover \emph{attribute implications}, the \emph{relative canonical base} and \emph{attribute exploration}.
This serves as a foundation for  \cref{sec:triadic-exploration}, where we look at \emph{implications in the triadic setting} and subsequently develop \emph{triadic exploration}.

\subsection{Formal Concept Analysis}
Formal Concept Analysis was introduced by Wille in \cite{Wille82}.
As the theory matured, Ganter and Wille compiled the mathematical
foundations of the theory in \cite{GanterWille1999}.
A \emph{formal context} $\context = \GMI$ consists of a set $G$ of objects, a set $M$ of attributes and an incidence relation
$I\subseteq G\times M$ with $(g,m)\in I$ meaning \emph{object $g$
  has attribute $m$}.
We define two derivation operators
$(\cdot)':\mathcal{P}(M)\rightarrow \mathcal{P}(G)$ and
$(\cdot)':\mathcal{P}(G)\rightarrow \mathcal{P}(M)$ in the following
way: For a set of objects
$A\subseteq G$, the set of \emph{attributes common to the
  objects in $A$} is provided by
  $A' \coloneqq \{ m \in M \mid \forall g\in A: (g,m)\in I \}.$
Analogously, for a set of attributes $B\subseteq M$, the set
of \emph{objects that have all the attributes from $B$} is provided by
  $B' \coloneqq \{ g \in G \mid \forall m\in B: (g,m)\in I \}.$
A \emph{formal concept} of a formal context $\context = \GMI$ is a
pair $(A,B)$ with $A\subseteq G$ and $B\subseteq M$ such that $A'=B$
and $A=B'$.  We call $A$ the \emph{extent} and $B$ the \emph{intent}
of the formal concept $(A,B)$.  The set of all formal concepts of a
context $\context$ is denoted by $\mathfrak{B}(\context)$.
Note that for any set $A \subseteq G$ the set $A'$ is the intent of a
concept and for any set $B\subseteq M$ the set $B'$ is the extent of a
concept.
%
%
The subconcept-superconcept relation on $\mathcal{B}(\context)$ is
formalized by:
$(A_1,B_1)\leq(A_2,B_2) :\Leftrightarrow A_1\subseteq A_2
(\Leftrightarrow B_1 \supseteq B_2)$.
The set of concepts together with this order relation
$(\mathfrak{B}(\context),\leq)$ forms a complete lattice, the
\emph{concept lattice}.
The vertical combination of two formal contexts
$\context_i=\GMII[i], i\in\{1,2\}$ without the same set of attributes $M$
is called the \emph{subposition} of $\context_1$ and
$\context_2$. Formally, it is defined as
$(\dot{G}_1 \cup \dot{G}_2,M,\dot{I}_1\cup \dot{I}_2)$, where
$\dot{G}_i:=\{i\}\times G$ and
$\dot{I}_i:= \{((i,g),(i,m))|(g,m)\in I_i\}$ for $i\in \{1,2\}$.
The \emph{subposition} of a set of contexts on the same set of attributes is defined analogously and we denote this by \emph{${subpos(\cdot)}$}.

\subsection{Triadic Concept Analysis}
Triadic Concept Analysis (TCA) was introduced by Lehmann and Wille in \cite{lehmann1995triadic} as an extension to Formal Concept Analysis with conditions. In particular they introduced the notion of \emph{triadic concepts} for which  Wille proceeded to show the basic theorem of triadic concept analysis in \cite{wille1995basic} -- clarifying the connection between triadic concepts and complete tri-lattices, analogous to the dyadic case.

The basic structure in TCA is a \emph{triadic context} which is similar to the formal context in FCA.
A \emph{triadic context} is defined as a quadruple $\T=\GMBY$, where $G,M$ and $B$ are sets and $Y\subseteq G\times M \times B$ is a ternary relation on these sets.
The elements of $G,M$ and $B$ are called objects, attributes and conditions respectively. For $g\in G$, $m\in M$ and $b\in B$ with $(g,m,b)\in Y$ we say that \emph{object $g$ has attribute $m$ under condition $b$}.
The conditions are understood in a broad sense~, cf. \cite{wille1995basic}: They comprise, amongst others, relations, interpretations, meanings, purposes and reasons concerning the connections of objects and attributes.

\begin{exmpl}
\label{ex:triadic-context}
The following example\footnote{The example is similar to the one given in \cite{ganter2016conceptual}, which inspired it.} will serve as our running example throughout the paper. 
It shows the situation of public transport at the train station Bf. Wilhelmshöhe with direction to the city center in Kassel.
From Bf. Wilhelmshöhe you can travel by one of four bus lines (52, 55, 100 and 500), four tram lines (1, 3, 4 and 7), one night tram (N3) and one regional tram (RT5) to the city center.
These are the objects $\Gex$ of our context.
The buses and trams leave the station at different times throughout the day.
The attributes $\Mex$ of our context are the aggregated leave-times, more specifically, we have split each day in five distinct time-slots: early morning (4:00 to 7:00), working hours (7:00 to 19:00), evening (19:00 to 21:00), late evening (21:00 to 24:00) and night (0:00 to 4:00).
The conditions $\Bex$ of our context are the days of the week.
A bus or tram line is related to a time-slot on a day if a bus or tram of this line leaves the station at least once during the time-slot on the day. This describes the ternary relation $Y\subseteq \Gex \times \Mex \times \Bex$.
We have aggregated Monday to Friday into a single condition, because the schedule is the same for these days.
Thus, we obtain the context $\Tex=(\Gex,\Mex,
\Bex,Y)$.
The resulting triadic context can be found in \cref{fig:triadic-context}.

\begin{figure}[t]
  \makebox[\textwidth][c]{
  \begin{minipage}{0.35\textwidth}
      \centering
      
 \resizebox{1.11\linewidth}{!}{\begin{tikzpicture}[ matrixstyle/.style={matrix of math nodes,
               nodes in empty cells, anchor=north east, fill=white,opacity=0.7}]
\matrix (Sun) [draw,matrixstyle]{
\phantom{\times} & \times & \times & \times & \phantom{\times}\\
\times & \times & \times & \times & \phantom{\times}\\
\times & \times & \times & \times & \phantom{\times}\\
\phantom{\times} & \phantom{\times} & \phantom{\times} & \phantom{\times} & \phantom{\times}\\
\phantom{\times} & \phantom{\times} & \phantom{\times} & \phantom{\times} & \phantom{\times}\\
\phantom{\times} & \phantom{\times} & \phantom{\times} & \phantom{\times} & \phantom{\times}\\
\phantom{\times} & \phantom{\times} & \phantom{\times} & \phantom{\times} & \times\\
\times & \times & \times & \times & \phantom{\times}\\
\phantom{\times} & \times & \times & \times & \phantom{\times}\\
\phantom{\times} & \times & \times & \times & \phantom{\times}\\
};\matrix (Sat) [draw,matrixstyle] at ($(Sun.north east) - (1,1)$){
\phantom{\times} & \times & \times & \times & \phantom{\times}\\
\times & \times & \times & \times & \phantom{\times}\\
\times & \times & \times & \times & \phantom{\times}\\
\phantom{\times} & \times & \phantom{\times} & \phantom{\times} & \phantom{\times}\\
\times & \times & \phantom{\times} & \phantom{\times} & \phantom{\times}\\
\phantom{\times} & \phantom{\times} & \phantom{\times} & \phantom{\times} & \phantom{\times}\\
\phantom{\times} & \phantom{\times} & \phantom{\times} & \phantom{\times} & \times\\
\times & \times & \times & \times & \phantom{\times}\\
\phantom{\times} & \times & \times & \times & \phantom{\times}\\
\times & \times & \times & \times & \phantom{\times}\\
};\matrix (Mo-Fr) [draw,matrixstyle] at ($(Sat.north east) - (1,1)$){
\times & \times & \times & \times & \phantom{\times}\\
\times & \times & \times & \times & \phantom{\times}\\
\times & \times & \times & \times & \phantom{\times}\\
\times & \times & \times & \phantom{\times} & \phantom{\times}\\
\times & \times & \times & \times & \phantom{\times}\\
\times & \times & \phantom{\times} & \phantom{\times} & \phantom{\times}\\
\phantom{\times} & \phantom{\times} & \phantom{\times} & \phantom{\times} & \phantom{\times}\\
\times & \times & \times & \times & \phantom{\times}\\
\times & \times & \times & \times & \times\\
\times & \times & \times & \times & \phantom{\times}\\
};\draw[dotted] (Sun.north west) -- (Mo-Fr.north west);
\draw[dotted] (Sun.north east) -- (Mo-Fr.north east);
\draw[dotted] (Sun.south east) -- (Mo-Fr.south east);
\node[left] at ($(Mo-Fr-7-1)+(-0.5,0)$) {N3};
\node[left] at ($(Mo-Fr-1-1)+(-0.5,0)$) {1};
\node[left] at ($(Mo-Fr-4-1)+(-0.5,0)$) {7};
\node[left] at ($(Mo-Fr-6-1)+(-0.5,0)$) {55};
\node[left] at ($(Mo-Fr-3-1)+(-0.5,0)$) {4};
\node[left] at ($(Mo-Fr-2-1)+(-0.5,0)$) {3};
\node[left] at ($(Mo-Fr-9-1)+(-0.5,0)$) {500};
\node[left] at ($(Mo-Fr-5-1)+(-0.5,0)$) {52};
\node[left] at ($(Mo-Fr-10-1)+(-0.5,0)$) {RT5};
\node[left] at ($(Mo-Fr-8-1)+(-0.5,0)$) {100};
\node[below,rotate=45,anchor=east] at ($(Mo-Fr-10-1)+(0,-0.5)$) {early-morning};
\node[below,rotate=45,anchor=east] at ($(Mo-Fr-10-2)+(0,-0.5)$) {working-hours};
\node[below,rotate=45,anchor=east] at ($(Mo-Fr-10-3)+(0,-0.5)$) {evening};
\node[below,rotate=45,anchor=east] at ($(Mo-Fr-10-4)+(0,-0.5)$) {late-evening};
\node[below,rotate=45,anchor=east] at ($(Mo-Fr-10-5)+(0,-0.5)$) {night};
\node[right] at ($(Sun-10-5)+(0.5,-0.25)$) {Sun};
\node[right] at ($(Sat-10-5)+(0.5,-0.25)$) {Sat};
\node[right] at ($(Mo-Fr-10-5)+(0.5,-0.25)$) {Mo-Fr};
\end{tikzpicture}}
\caption{Triadic context \\$\Tex$ of \cref{ex:triadic-context}}
  \label{fig:triadic-context}
  \end{minipage}
  \begin{minipage}{0.65\textwidth}
    \centering
    \resizebox{1.0\linewidth}{!}{\input{context-family}}
  \caption{The triadic context $\Tex$ from \cref{ex:triadic-context} represented as context family of the condition contexts $\K_{\text{Mo-Fr}}$, $\K_{\text{Sat}}$ and $\K_{\text{Sun}}$}
  \label{fig:triadic-context-as-context-family}
\end{minipage}
}
\end{figure} 
\end{exmpl}

Naturally, we can view the triadic context as a family of formal contexts, where each context represents one condition, basically slicing the triadic context vertically along the conditions.
In \cref{fig:triadic-context-as-context-family} we provide the resulting context family of our running example. 

Formally such a family of contexts representing a triadic context $\T=\GMBY$ is a set of contexts $\K_b,\ b\in B$ where $\K_b:= (G,M,I_b)$ with $(g,m)\in I_b :\Leftrightarrow (g,m,b)\in Y$.  We will refer to the contexts $\K_b$ as \emph{condition contexts} of the triadic context $\T$; for our example these are $\K_{\text{Mo-Fr}}$, $\K_{\text{Sat}}$ and $\K_{\text{Sun}}$.

\subsection{Attribute Implications}
Attribute implications are used to describe dependencies between attributes in a formal context.
In the following we give a brief introduction.
Let $M$ be a set of attributes.
(For a start, we do not require it to be related to a specific context.)
An \emph{attribute implication} over $M$ is a pair of subsets $A,B\subseteq M$ of $M$. We denote this by $A\implies B$. We call $A$ the \emph{premise} and $B$ the \emph{conclusion} of the implication $A\implies B$.

We denote the set of all implications over a set $M$ by $\Impm  = \{A\implies B| A,B \subseteq M\}$.

A subset $T\subseteq M$ \emph{respects} an attribute implication $A\implies B$ over $M$ if $A\not\subseteq T$ or $B\subseteq T$.  We then also call $T$ a \emph{model} of the implication.
$T$ \emph{respects a set} $\mathcal{L}$ of implications if $T$ respects all implications in $\mathcal{L}$.
An implication $A\implies B$ \emph{holds} in a set of subsets of $M$ if each of these subsets respects the implication.

For a formal context $\context=\GMI$ we say that an implication $A\implies B$ over $M$ \emph{holds in the context} if for every object $g\in G$ the object intent $g'$ respects the implication.
We then also call $A\implies B$ a \emph{valid implication} of $\context$.
An implication $A\implies B$ holds in $\K$ if and only if every object $g\in G$ that has all attributes in $A$ also has all attribute in $B$.
Further, an implication $A\implies B$ holds in $\context$ if and only if $B\subseteq A''$, or equivalently $B' \subseteq A'$.
An implication $A\implies B$ \emph{follows} from a set $\L$ of implications over $M$ if each subset of $M$ respecting $\L$ also respects $A\implies B$.  A family of implications is called \emph{closed} if every implication following from $\L$ is already contained in $\L$.  Closed sets of implications are also called \emph{implication theories}.

\subsubsection{Relative Canonical Base.}
The set of all implications that hold in a given context $\context$ have a canonical irredundant representation which is called the \emph{canonical base}, cf. \cite{GanterWille1999,guigues1986familles}. Stumme has generalized this representation to the case where some (background) implications are known \cite{stumme96attribute}, i.e. attribute implications that are known to hold based on prior knowledge. 

Given an formal context $\context=\GMI$ and a set of (background) implications $\mathcal{L}_0$ on $M$ that hold in the context $\context$. A \emph{pseudo-intent} of $\context$ \emph{relative to $\L_0$} is a set $P\subseteq M$
where $P$ respects $\L_0$, $P\not = P''$ and if $Q\subseteq P,\ Q \not = P$, is a relative psuedo-intent of $\K$ then $Q''\subseteq P$.
The set $\L_{\K,\L_0} := \{P\implies P''| P \text{ relative pseudo-intent of } \K\}$ is called the \emph{canonical base} of $\K$ \emph{relative to $\L_0$}, or simply the \emph{relative canonical base}.
All implications in $\L_{\K,\L_0}$ hold in \K.
\begin{thm}[see \cite{GanterObiedkov04Triadic,stumme96attribute}]\label{thm:relative-canonical-base}
  If all implications of $\L_0$ hold in $\K$, then
  \begin{enumerate}
  \item each implication that holds in $\K$ follows from $\L \cup \L_0$, and
  \item $\L_{{\K,\L_0}}$ is irredundant w.r.t. 1.
  \end{enumerate}
\end{thm}
The notion of a relative canonical base combined with \cref{thm:relative-canonical-base} allows us to reduce the amount of questions that need to be posed during a triadic exploration.

\subsection{Attribute Exploration}
Attribute exploration~(\cite{ganter1984two}, cf. also \cite{ganter2016conceptual,GanterWille1999}) is a knowledge acquisition method based on a question-answer scheme to obtain the implication theory of a domain. 

Let us consider a domain (a formal context) $\GMI$ that we do not know completely and that we want to explore and a domain expert for this domain.  We start with a (possibly empty) set of known (background) implications $\L$ and a (possibly empty) set $G_E\subseteq G$ of known objects, represented as (possibly empty) formal context $\context[E]=\GMII[E]$.  In every step of the attribute exploration we have a set of already accepted implications $\L$ and a context of already provided counterexamples $\context[E]$.  The attribute exploration algorithm picks the next implication $A\implies B$ that does not follow from $\L$ and that holds in $\context[E]$.  It then asks the expert whether the implication truly holds in the domain.  The expert can either confirm that the implication holds or they can refute its validity by providing a counterexample, i.e., an object $g\in G$ whose intent does not respect the implication.  If the expert confirms the implication's validity in the domain, it is added to the set $\L$, otherwise the provided counterexample is added to the context of counterexamples $\context[E]$.  This process is repeated until there is no implication left to be asked.

After performing the attribute exploration we have the canonical base of implications
from which every valid implication in the domain follows. Furthermore, for every implication that is not valid, the set of examples contains a counterexample.

\section{Triadic Exploration}
\label{sec:triadic-exploration}
In this section we look at \emph{implications} in the triadic setting, in particular, we formally introduce \emph{conditional attribute implications}, and develop a \emph{triadic exploration} for Triadic Concept Analysis as proposed by Ganter and Obiedkov in \cite{ganter2016conceptual,GanterObiedkov04Triadic}.

\subsection{Conditional Attribute Implications}
In formal contexts (of type $\GMI$) the matter of \emph{implications} is fairly straightforward: There are \emph{attribute implications} to describe dependencies between attributes (and dually there are \emph{object implications}).
In triadic contexts, the notion of implication is not as simple.  This manifests in a multitude of types of implications that have been proposed:
The earliest suggestion for a \emph{triadic implication} came from Biedermann \cite{Biedermann98Foundation}, where he suggested the study of implications of the form $(R\implies S)_C$ which is interpreted as: \textit{If an object has all attributes from R under all conditions from C, then it also has all attributes from S under all conditions from C}.

In  \cite{GanterObiedkov04Triadic}, Ganter and Obiedkov studied some other types of implications for the triadic setting.  They introduced a stronger version of the triadic implication called \emph{conditional attribute implications} to describe dependencies that hold for some conditions.  The symmetry arising from the arbitrary choice of objects, attributes and conditions in a triadic context results in five more types of implications.  Further, they introduced another generalization of Biedermann's triadic implication called \emph{attribute$\times$condition implication} to express dependencies between combinations of attributes and conditions. 
For the remainder of this paper we will focus on \emph{conditional attribute implications}, because they best serve our goal of developing attribute exploration with multiple experts.

Given a triadic context $\T=\GMBY$,
a \emph{conditional attribute implication} is an expression of the form $R \cimplies{C} S$ where $R,S \subseteq M$, $C\subseteq B$, which reads as: \emph{$R$ implies $S$ under all conditions from $C$}.
A conditional attribute implication $R\cimplies{C}S$ \emph{holds} in a triadic context $\T$ iff for each condition $c\in C$ it holds that if an object $g\in G$ has all the attributes in $R$ it also has all the attributes in $S$.  This is the case if the implication $R\implies S$ holds in every conditional context $\K_c$ for $c\in C$.

\begin{prop}
  Let $\T=\GMBY$ and $\K_c$, $c\in B$, its respective condition contexts. For a conditional implication $R\cimplies{C}S$ with $R,S\subseteq M$ and $C\subseteq B$, the following statements are equivalent:
  \begin{enumerate}
  \item $R\cimplies{C}S$ holds in $\T$
  \item $R\implies S$ holds in $\K_c$ for every $c\in C$
  \item $R\implies S$ holds in $subpos(\{\K_c|c \in C\})$
  \end{enumerate}
\end{prop}
\begin{proof}
  $1. \Leftrightarrow 2.$  follows directly from the definitions of \emph{holds} in the triadic and dyadic setting.
  $2. \Leftrightarrow 3.$ follows from the definition of subposition and that an implication $R\implies S$ holds in a context if and only if for every object $g$ the object intent respects the implication.
\end{proof}
\begin{exmpl}
  \label{ex:conditional-implications}
  In the context family of \cref{ex:triadic-context} in \cref{fig:triadic-context-as-context-family} we observe that the implication \emph{$\text{early-morning}\implies \text{working-hours}$} holds in all three condition contexts $\K_{\Mo}, \K_{\Sat}$ and $\K_{\Sun}$, hence,
$\emph{\text{early-morning}}\cimplies{\Mo,\Sat,\Sun}\emph{\text{working-hours}}$
holds in $\Tex$.
Whereas the implication \emph{$\text{working-hours}\implies \text{evening}$} only holds in the condition context $\K_{\Sun}$  because tram line $7$ is a counterexample in $\K_{\Sat}$ and bus line $55$ is a counterexample in $\K_{\Mo}$ and thus
$\emph{\text{working-hours}}\cimplies{\Sun}\emph{\text{evening}}$ holds in $\Tex$, but
$\emph{\text{working-hours}}\cimplies{\Mo,\Sat,\Sun}\emph{\text{evening}}.$ does not.
\end{exmpl}

Clearly, if a conditional implication $R\cimplies{C} S$ holds in a triadic context $\T$ then all conditional implications $R\cimplies{D}S$ with $D\subseteq C$ hold as well.  
Further, for every subset $C\subseteq B$ there is a set of conditional implications $R\cimplies{C}S$ that hold in $\T$.
This set of conditional implications for a fixed set of conditions $C$ is the implication theory of the subposition of condition contexts $subpos(\{\K_c| c\in C\})$.

\subsubsection{Context of Conditional Implications.}
A nice way to structure the conditional implications that hold in a triadic context $\T$ is to use the approach suggested by Ganter and Obiedkov, cf.~\cite{GanterObiedkov04Triadic}, and to introduce a \emph{context of conditional implications}:
Given a triadic context $\T$, we construct a formal context $\mathcal{C}_{imp}(\T) := (\Impm,B,I)$, where the set of all possible implications on $M$ is the object set, the set of conditions $B$ of the triadic context $\T$ is the set of attributes and the incidence relation $I$ is determined by
$$
(R\implies S) I c \ \vcentcolon\Leftrightarrow R\cimplies{c}S \text{ holds in } \T.
$$
 
The formal concepts of $\mathcal{C}_{imp}(\T)$ are pairs $(\L,\ C)$, where
 $\L$ is a set of implications and $C$ is a set of conditions, such that $\L$ is the set of all implications $R\implies S$ for which $R\cimplies{C}S$ holds, and $C$ is the largest set of conditions for which this is the case.
%
 These concepts structure the set of conditional implications in a lattice ordered by the conditions for which they hold.
 Their extents form a system of implication theories.

\begin{exmpl}
For our running example we present the concept lattice of $\mathcal{C}_{imp}(\Tex)$ with simplified labels in \cref{fig:lattice-conditional-implications-example}:
The extent of the top node always contains the implications that hold under the empty set of conditions, i.e., the whole set $\Impm$. We omit this label.
For the other nodes we give the relative canonical base with respect to set of implications from all nodes below.
\begin{figure}[t]
  \centering 
  \resizebox{1.0\linewidth}{!}{ \colorlet{mivertexcolor}{black}
\colorlet{jivertexcolor}{black}
\colorlet{vertexcolor}{mivertexcolor!50}
\colorlet{bordercolor}{black!80}
\colorlet{linecolor}{gray}
\tikzset{vertexbase/.style={semithick, shape=circle, inner sep=2pt, outer sep=0pt, draw=bordercolor},%
  vertex/.style={vertexbase, fill=vertexcolor!45},%
  mivertex/.style={vertexbase, fill=mivertexcolor!45},%
  jivertex/.style={vertexbase, fill=jivertexcolor!45},%
  divertex/.style={vertexbase, top color=mivertexcolor!45, bottom color=jivertexcolor!45},%
  conn/.style={-, thick, color=linecolor}%
}
\begin{tikzpicture}
  \begin{scope} 
    \begin{scope} 
      \foreach \nodename/\nodetype/\xpos/\ypos in {%
        0/vertex/0.124/2.4,
        1/divertex/-1.0/3.35,
        2/divertex/1.5/4.0,
        3/divertex/-1.0/4.65,
        4/vertex/0.124/5.6
      } \node[\nodetype] (\nodename) at (\xpos, \ypos) {};
    \end{scope}
    \begin{scope} 
      \path (1) edge[conn] (3);
      \path (3) edge[conn] (4);
      \path (0) edge[conn] (2);
      \path (2) edge[conn] (4);
      \path (0) edge[conn] (1);
    \end{scope}
    \begin{scope} 
      \foreach \nodename/\labelpos/\labelopts/\labelcontent in {%
        0/below//{\parbox[c]{9.0cm}{early-morning $\implies$ working-hours\\ working-hours, night $\implies$ late-evening, early-morning, evening\\ late-evening $\implies$ working-hours, evening\\ evening $\implies$ working-hours}},
        1/below left//{\parbox[c]{3.5cm}{evening $\implies$ working-hours, late-evening}},
        1/above left//{Sat},
        2/below right//{\parbox[c]{5.0cm}{working-hours $\implies$ early-morning\\ night $\implies$ working-hours, late-evening,\\\phantom{night $\implies$} early-morning, evening 
        }},
        2/above right//{Mo-Fr},
        3/below left//{\parbox[c]{3.5cm}{working-hours $\implies$ evening}},
        3/above left//{Sun}
      } \coordinate[label={[\labelopts]\labelpos:{\labelcontent}}](c) at (\nodename);
    \end{scope}
  \end{scope}
\end{tikzpicture}}
  \caption{The lattice of conditional implications of the running example $\Tex$ with simplified labels, which consist of the relative canonical base with respect to the implications in all nodes below. We omit the top label of implications as the extent of this concept is always $\Impm$.}
  \label{fig:lattice-conditional-implications-example}
\end{figure}
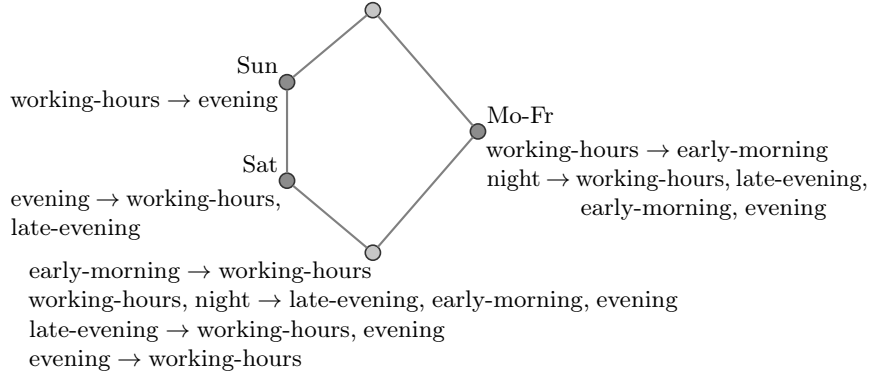
%
Looking at the implications from \cref{ex:conditional-implications}, we find the implication
  $\emph{\text{early-morning}}\implies\emph{\text{working-hours}}$
  at the \emph{bottom node}, because it holds for all three conditions, whereas we find the implication
  $\emph{\text{working-hours}}\implies \emph{\text{evening}} $
  at the node for \emph{Sunday}, because that is the only condition for which it holds.
\end{exmpl}

\subsection{Triadic Exploration}
Now, we develop \emph{Triadic Exploration} to explore the conditional implications of a triadic domain.

Previously, we have structured the conditional implications of a triadic domain $\T$ as a system of implication theories by utilizing the context of conditional implications $\mathcal{C}_{imp}(\T)$. This was possible because we had complete information about the domains implications in the context $\T$.
However, it is easy to imagine a situation where we can access the information about a domain only indirectly through a domain expert and where an attribute exploration might be useful.
For our running example, imagine someone with a bus and train schedule where the information can be looked up but is not fully available at once.
Now the question is: How to explore the complete system of conditional implications?



A naive approach is to explore the implication theory for each fixed subset of the conditions, essentially exploring each node of the system of implication theories independently.
But, this is clearly not a good idea; it means answering many questions multiple times for each condition.

A better approach might be to only explore the implication theory for every condition, each providing one column in the context of conditional implications $\mathcal{C}_{imp}$. Then we can compute the concept lattice without any further interactions with the expert.


However, there are some points to consider that suggest a different approach, cf. \cite{ganter2016conceptual}:
First, to stay in the triadic setting, a complete counterexample to a question should
describe the new object by the attributes it has \emph{for each of the conditions}, and not only for the one, that is currently under consideration.
And second, some implications may hold for several conditions and the domain expert might want to confirm each of them for multiple conditions at once.

Thus, we come back to the context of conditional implications.
Ganter and Obiedkov suggested to explore the triadic domain by exploring the nodes in the lattice of conditional implications from the bottom up; using the already known valid implications as background knowledge.
Hence, as we explore the system of conditional implications, we successively fill the context of conditional implications.
 
In the following we describe the nested process of exploring the nodes of the concept lattice of conditional implications with the help of two algorithms: \cref{alg:explore-conditions} for the exploration of the conditional implications for a fixed set of conditions and \cref{alg:triadic-exploration} that uses this algorithm as a subroutine to explore all conditional implications of the triadic domain.

\subsubsection{Explore Conditional Implications for a fixed set of Conditions.}

For a fixed set of conditions $D\subseteq B$ in a triadic domain $\T=\GMBY$, the exploration algorithm is an adapted version of the algorithm for attribute exploration with background implications and exceptions, see \cite{ganter2016conceptual,stumme96attribute}.
In \cref{alg:explore-conditions} we present an implementation for the exploration in pseudo-code.

The algorithm starts with some background knowledge, in particular: A triadic context $\E=(G_E,M,B,Y_E)$, that contains some examples from the domain $\T$, and a set of implications $\L_0$ that are known to hold for all conditions in $D$; both of these can be empty.
The rest of the domain can only be accessed by the algorithm through interaction with the domain expert.
In each step, the algorithm determines the next implication $A\implies A''$ to ask the expert.
To determine the next question $A\implies A''$ the algorithm uses both the information from the examples in $\E$ and the known valid implications in $\L$.
It automatically skips questions that follow from the implications in $\L$ or for which $\E$ already contains a counterexample.
More precisely, $A$ is the next \emph{relative pseudo-intent} in $subpos(\{\K_d|d\in D\})$, i.e., the lectically smallest set $A$ closed under the set of known valid implications and background implications $\L$ that is not already closed in the subposition context of examples for the conditions in $D$.

Essentially, this algorithm is an attribute exploration with background implications on the subposition of the condition contexts.  Additionally, it tracks which implications hold for which conditions in $D$. This enables us to reduce the amount of interaction required from the expert in subsequent explorations by preventing to ask the same question multiple times for different subposition contexts.
The proof of correctness for \cref{alg:explore-conditions} is a straightforward adaption of the proof of \cite[Theorem 6]{stumme96attribute} and we therefore omit the details.

Note that we chose to collect all implications that are asked about and the subset of conditions of $D$ for which they hold in \cref{alg:A1L13} instead of only adding the implications that hold for all conditions in the context $\C$.
Hence, if there is a counterexample, i.e., the implication does not hold for $D$, we track for which subset of $D$ (if any) the implication does hold.
This further reduces the number of questions posed in later explorations.  The trade-off is that the background knowledge we have is not just of nodes below the currently explored one in the lattice but may also contain implications that first hold for the conditions of the current node.  This has no effect on the implication theory of the node but somewhat complicates the labeling of the node -- we cannot simply use the relative canonical base with respect to the knowledge we have.
In contrast, if we only added the implications that hold for all conditions in the current exploration then the labels are exactly the implications of the relative canonical base, but, we might have to ask some questions multiple times for some of the conditions.  For our running example this approach further reduces the number of questions posed to the expert from fifteen to twelve, cf. \cref{ex:triadic-exploration-running-example} in \cref{subsec:triadic-exploration-running-example}.
  
\begin{algorithm}[t]
    \small \SetKwComment{Comment}{}{} \SetKw{Kwin}{in} 
    \DontPrintSemicolon \SetAlgoLined
    \SetKwInOut{KWInteractive}{Interactive Input}
    \SetKwComment{ic}{}{}
    \KwIn{
\parbox[t]{10cm}{
      a set of conditions $D\subseteq B$,
      a triadic context $\E=(G_E,M,B,Y_E)$ of examples (possibly empty) and 
      a set $\L_0$ of background implications known to hold for all conditions in $D$ (also possibly empty)}
    }
    \KWInteractive{\parbox[t]{8.25cm}{
      $\ (\star)$ The expert confirms or rejects an implication to hold for the set of conditions $D$. Upon rejection the expert provides a counterexample $g$ from the domain together with its relation to all conditions and all attributes, i.e., the context $\K_g:=(M,B,I)$ where $(m,b)\in I \Leftrightarrow$ $g$ has $m$ under the condition $b$ in the domain.}
    }
    \KwOut{
\parbox[t]{9.7cm}{
      the relative canonical base $\L\setminus\L_0$ of implications that hold for all conditions in $D$ with respect to $\L_0$, a possibly enlarged triadic context of counterexamples $\E$ and the formal context $\C$ of asked implications and the conditions for which they hold.}}
   $\L:=\L_0$\;
   $A:=\emptyset$\;
   $\C:=(\emptyset,B,\emptyset)$\;
   \While{$A \not = M$}
   {
     \While{ $A\not = A''$ in $S$ where\;
       $S:= (G_S,M,J) = $ subposition of $\K_d$ for $d\in D$ with\;
       $K_d:= (G_E,M,I_d)$ where $(g,m)\in I_d \Leftrightarrow (g,m,d)\in Y_E$\;
      }{
     Ask the expert if $A\implies A''$ holds for all conditions $d\in D$\ic*{$(\star)$}
     \lIf{$A\implies A''$ holds}{$\L := \L \cup \{A\implies A''\}$}
     \lElse{extend $\E$ with the counterexample provided by the expert\ic*{$(\star)$}}
     extend $\C$ with the object $A\implies A''$ and its relation to all conditions $d\in D$ \label{alg:A1L13}\ic*{$(\star)$}
   }
   $A:= \operatorname{NextClosure}(A,M,\L)$ \tcc*{computes the next closure of $A$ in $M$ with respect to the implications in $\L$; see for example \cite{GanterWille1999,ganter2016conceptual}}
 }
 \KwRet $\L\setminus \L_0$, $\E$ and $\C$\;
 \caption{explore-conditions}
 \label{alg:explore-conditions}
\end{algorithm}

\subsubsection{The Order of Explorations.}
To determine the sequence in which the nodes of the lattice of conditional implications are explored, Ganter and Obiedkov further suggested to follow a linear extension of the lattice of conditional implications, see~\cite{GanterObiedkov04Triadic}, and later specified this to follow the \emph{NextExtent-Algorithm}, i.e., \emph{NextClosure} on the extents, on the context of conditional attribute implications, see~\cite{ganter2016conceptual}.

However, in our setting the \emph{NextExtent-Algorithm} does not fit.
The problem is that we may not have the necessary information to correctly determine the next node to explore.\footnote{For the same reason, the nested application of NextClosure for computing all concepts of a triadic context, as described in \cite{trias2006jaeschke,JaeschkeHothoEtAl08jws}, cannot serve as a base for the triadic exploration.} 
This is because the questions that are asked during the exploration of a node are not guaranteed to discriminate between the conditions that are being explored.  Questions that would discriminate between conditions are not asked if there already exists a counterexample for any one of the conditions.
This might result in not exploring all nodes of the lattice.
\begin{exmpl}
Let us illustrate the problem with a small example:
Take a look at the domain given by the triadic context $\T_1$ in \cref{ex:counterex-context-family}.
If we explore this context and begin with the bottom node, i.e., the implications that hold for all conditions without any background knowledge, then
the first question posed to the expert is \emph{$\emptyset \implies ab?$}, which the expert refutes with a counterexample -- object $1$ with all its attributes under all conditions.
It substantiates that implication holds for neither of the conditions.
The second question that is posed to the expert is $b\implies a?$ which the expert confirms.
This concludes the exploration of the bottom node in this example.
If we now compute the next extent in the resulting context of conditional implications $\C$ in order to determine which node to explore next, we obtain $\operatorname{NextExtent}(\emptyset) = \{b\implies a\}$ with intent $\{d_1,d_2\}$ which we just explored and then $\operatorname{NextExtent}(\{b\implies a\}) = G_{\C}$ with intent $\emptyset$ which concludes the exploration.
However, clearly the implication $\emptyset\implies a$ holds in $d_1$ but not in $d_2$ and is missing in $\C$.
And, the question  $\emptyset\implies a?$ was not posed to the expert because there already existed a counterexample for condition $d_2$ after the first question.
Similarly, the implication $a\implies b$ holds in $d_2$ but not in $d_1$ and is also missing.  In \cref{fig:counterex-figure}, we present both the lattice of $\C$ and the lattice of $\mathcal{C}_{imp}(\T_1)$.
Hence, an exploration that uses the \emph{NextExtent-Algorithm} to determine which nodes of the conditional implications lattice to explore next does not necessarily explore all nodes of the lattice.
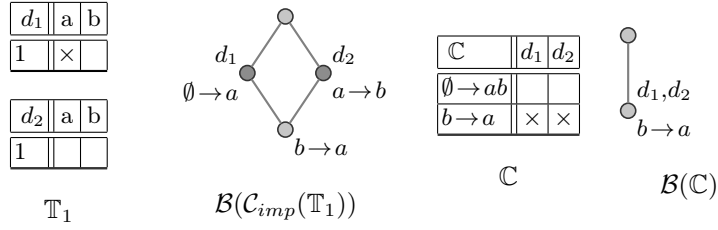
\begin{figure}[t]
\centering
\begin{minipage}{.25\textwidth}
  \centering
\begin{cxt}%
  \cxtName{$d_1$}%
  \cxtNichtKreuz{}%
  \att{a}%
  \att{b}%
  \obj{x.}{1}
\end{cxt}

\vspace{1em}
\begin{cxt}%
  \cxtName{$d_2$}%
  \cxtNichtKreuz{}%
  \att{a}%
  \att{b}
  \obj{..}{1}
\end{cxt}
\captionof*{figure}{$\T_1$}
\end{minipage}%
\begin{minipage}{0.24\textwidth}
  \colorlet{mivertexcolor}{black}
\colorlet{jivertexcolor}{black}
\colorlet{vertexcolor}{mivertexcolor!50}
\colorlet{bordercolor}{black!80}
\colorlet{linecolor}{gray}
\tikzset{vertexbase/.style={semithick, shape=circle, inner sep=2pt, outer sep=0pt, draw=bordercolor},%
  vertex/.style={vertexbase, fill=vertexcolor!45},%
  mivertex/.style={vertexbase, fill=mivertexcolor!45},%
  jivertex/.style={vertexbase, fill=jivertexcolor!45},%
  divertex/.style={vertexbase, top color=mivertexcolor!45, bottom color=jivertexcolor!45},%
  conn/.style={-, thick, color=linecolor}%
}
  \begin{tikzpicture}
  \begin{scope} 
    \begin{scope} 
      \foreach \nodename/\nodetype/\xpos/\ypos in {%
        0/vertex/1/0,
        1/divertex/0.5/0.75,
        2/divertex/1.5/0.75,
        3/vertex/1/1.5
      } \node[\nodetype] (\nodename) at (\xpos, \ypos) {};
    \end{scope}
    \begin{scope} 
      \path (0) edge[conn] (1);
      \path (0) edge[conn] (2);
      \path (1) edge[conn] (3);
      \path (2) edge[conn] (3);
    \end{scope}
    \begin{scope} 
      \foreach \nodename/\labelpos/\labelopts/\labelcontent in {%
        0/below right//{$b \implies a$},
        1/below left//{$\emptyset \implies a$},
        1/above left//{$d_1$},
        2/below right//{$a\implies b$},
        2/above right//{$d_2$}
      } \coordinate[label={[\labelopts]\labelpos:{\labelcontent}}](c) at (\nodename);
    \end{scope}
  \end{scope}
\end{tikzpicture}
\captionof*{figure}{$\mathcal{B}(\mathcal{C}_{imp}(\T_1))$}
\end{minipage}%
\begin{minipage}{.24\textwidth}
  \centering
  \begin{cxt}%
  \cxtName{$\C$}%
  \cxtNichtKreuz{}%
  \att{$d_{1}$}%
  \att{$d_{2}$}
  \obj{..}{$\emptyset \implies ab$}
  \obj{xx}{$b \implies a$}
\end{cxt}
\captionof*{figure}{$\C$}
\end{minipage}%
\begin{minipage}{0.15\textwidth}
  \colorlet{mivertexcolor}{black}
\colorlet{jivertexcolor}{black}
\colorlet{vertexcolor}{mivertexcolor!50}
\colorlet{bordercolor}{black!80}
\colorlet{linecolor}{gray}
\tikzset{vertexbase/.style={semithick, shape=circle, inner sep=2pt, outer sep=0pt, draw=bordercolor},%
  vertex/.style={vertexbase, fill=vertexcolor!45},%
  mivertex/.style={vertexbase, fill=mivertexcolor!45},%
  jivertex/.style={vertexbase, fill=jivertexcolor!45},%
  divertex/.style={vertexbase, top color=mivertexcolor!45, bottom color=jivertexcolor!45},%
  conn/.style={-, thick, color=linecolor}%
}
  \begin{tikzpicture}
  \begin{scope} 
    \begin{scope} 
      \foreach \nodename/\nodetype/\xpos/\ypos in {%
        0/vertex/1/0,
        1/vertex/1/1
      } \node[\nodetype] (\nodename) at (\xpos, \ypos) {};
    \end{scope}
    \begin{scope} 
      \path (0) edge[conn] (1);
    \end{scope}
    \begin{scope} 
      \foreach \nodename/\labelpos/\labelopts/\labelcontent in {%
        0/below right//{$b \implies a$},
        0/above right//{$d_1,d_2$}
      } \coordinate[label={[\labelopts]\labelpos:{\labelcontent}}](c) at (\nodename);
    \end{scope}
  \end{scope}
\end{tikzpicture}
\captionof*{figure}{$\mathcal{B}(\C)$}
\end{minipage}%
\caption{
  A triadic context $\T_1$,
  the lattice of conditional implications of $\T_1$,
  the context $\C$ after exploring the conditional implications of $\T_1$ using the \emph{NextExtent-Algorithm} to determine the next conditions to explore,
  and the lattice of $\C$ }
\label{ex:counterex-C-lattice}
\label{ex:counterex-C-imp}
\label{ex:counterex-context-family}
\label{ex:counterex-C-imp-lattice}
\label{fig:counterex-figure}
\end{figure}
\end{exmpl}

To circumvent this problem, we use the suggested strategy of exploring the lattice node by node from the bottom up with the already known valid conditional implications in $\mathcal{C}_{imp}$ as background knowledge.  But, instead of using the \emph{NextExtent-Algorithm} to incrementally determine the next combination of conditions to explore, we simply follow a linear extension of $(\mathcal{P}(B)\setminus \emptyset, \supseteq)$.
Which means, we walk through all subsets of $B$ sorted by their cardinality from biggest to smallest and stop when we have explored all subsets of cardinality one.
At first glance this might look as if we explore more nodes than necessary, because the implication theory of a condition might be included in another ones and thus is explored at least twice -- once in combination and once alone. But, because we only ask questions about implications that are unknown with respect to the knowledge we already have when the condition is explored alone, these questions won't be asked again.

In \cref{alg:triadic-exploration} we present the algorithm for \emph{triadic exploration} in pseudo-code:
We walk through $(\mathcal{P}(B)\setminus \emptyset, \supseteq)$, i.e. the subsets of conditions, in \cref{alg:A2L1}.
For each set of conditions $D\subseteq B$ we determine the implications $\L$ that are  known to hold for all conditions in $D$ in \cref{alg:A2L2}.
We compute the canonical base relative to $\L$ in the subposition of condition context of $D$ \cref{alg:A2L3}.
And, update the known examples $E$ and the known implications in \cref{alg:A2L4,alg:A2L5}.

\begin{algorithm}[t]
    \small \SetKwComment{Comment}{}{} \SetKw{Kwin}{in}
    \DontPrintSemicolon \SetAlgoLined
    \SetKwInOut{KWInteractive}{Interactive Input}
    \SetKwComment{ic}{}{}
    \KwIn{a triadic context $\E=(G_E,M,B,Y_E)$ of examples (possibly empty) and
     a context $\C=(G_{\C},B,I_{\C})$ of implications known to hold for some conditions}
    \KwOut{a triadic context of counterexamples $\E$ and the context of conditional implications $\C$, from which all valid conditional implications can be inferred}

    \For {$D$ in linear extension of $(\mathcal{P}(B)\setminus \emptyset, \supseteq)$\label{alg:A2L1}}{
      $\L:= D'$ (in $\C$)\; \label{alg:A2L2}
      $\L_D,\E_D,\C_D := \operatorname{explore-conditions}(D,\E,\L)$\;\label{alg:A2L3}
      $\E:=\E_D$\;\label{alg:A2L4}
      $\C:= \C\cup \C_D = (G_{\C}\cup G_{\C_D},B, I_{\C}\cup I_{\C_D})$\;\label{alg:A2L5}   
    }
  
    \KwRet $\E$ and $\C$ \;
    \caption{triadic-exploration}
    \label{alg:triadic-exploration}
  \end{algorithm}

  \subsection{An Example for Triadic Exploration}
  \label{subsec:triadic-exploration-running-example}
  \begin{exmpl}\label{ex:triadic-exploration-running-example}
  We now give a brief example for a triadic exploration of the domain of our running example (\cref{ex:triadic-context}):
  Let us assume we only have a triadic expert for this domain and not the whole domain information -- imagine someone with access to a search interface for the bus and train schedule of \cref{fig:triadic-context-as-context-family}.
  In \cref{tbl:conditional-exploration} we have listed all interactions with the expert. Each row shows one interaction and the order of interactions is from top to bottom.
  The resulting lattice of conditional implications is exactly the lattice shown in \cref{fig:lattice-conditional-implications-example}.  The extent of each concept of this lattice is a generating set for the implication theory of implications that hold for all conditions of the intent which follows from \cref{thm:relative-canonical-base}, and, because we iteratively computed relative canonical bases. Thus, we know that, for each concept , the implications in its extent are complete, but -- as a union of ``stacked'' relative canonical bases -- not necessarily irredundant.
  \end{exmpl}
  
\begin{figure}[t]
  \footnotesize
\resizebox{1\linewidth}{!}{
\begin{tabular}{l|l|l|l}
Conditions & Question -- does the implication hold? & Question holds for & Answer\\
  \hline
Mo-Fr, Sat, Sun &	 $\emptyset $$\ \implies \ $working-hours, late-evening, early-morning, night, evening &	 $\emptyset $  &	 RT5 \\
Mo-Fr, Sat, Sun &	 $\emptyset $$\ \implies \ $working-hours, late-evening, evening &	 $\emptyset $  &	 52 \\
Mo-Fr, Sat, Sun &	 evening$\ \implies \ $working-hours, late-evening &	 Sat, Sun &	 7 \\
Mo-Fr, Sat, Sun &	 evening$\ \implies \ $working-hours &	 Mo-Fr, Sat, Sun &	 true \\
Mo-Fr, Sat, Sun &	 night$\ \implies \ $working-hours, late-evening, early-morning, evening &	 Mo-Fr &	 N3 \\
Mo-Fr, Sat, Sun &	 early-morning$\ \implies \ $working-hours &	 Mo-Fr, Sat, Sun &	 true \\
Mo-Fr, Sat, Sun &	 late-evening$\ \implies \ $working-hours, evening &	 Mo-Fr, Sat, Sun &	 true \\
Mo-Fr, Sat, Sun &	 working-hours, night$\ \implies \ $late-evening, early-morning, evening &	 Mo-Fr, Sat, Sun &	 true \\
Mo-Fr, Sun &	 working-hours$\ \implies \ $evening &	 Sun &	 55 \\
Mo-Fr, Sat &	 working-hours, evening$\ \implies \ $early-morning &	 Mo-Fr &	 500 \\
Sun &	 working-hours, late-evening, early-morning, evening$\ \implies \ $night &	 $\emptyset $  &	 4 \\
Mo-Fr &	 working-hours$\ \implies \ $early-morning &	 Mo-Fr &	 true \\
\end{tabular}
}
\caption{Triadic Exploration of the running example.  Each row represents one interaction with the expert.  It comprises the set of conditions that is explored, the question posed in form of an implication, the conditions for which the implication holds, and, the answer given by the expert.}

\label{tbl:conditional-exploration}
\end{figure}


\section{Application for Exploration with multiple Experts}
\label{sec:appl-mult-experts}

In this section we discuss how to adapt triadic exploration to a setting where we have multiple experts with different views on a domain (i.e., a set of attributes).  In \cref{sec:introduction}, we have briefly discussed the problem of exploration with multiple experts with different views and concluded that combining answers from different experts is not a good strategy for attribute exploration in general.  We have further established that all previous methods for multi-expert exploration avoided this problem by assuming that the experts' knowledge is derived from some consistent domain knowledge.

Here, we suggest a different approach that allows for a group of experts with different, opposing views on a domain.  The basic idea is to accept all answers equally and look for the subset of knowledge that all experts agree on. To explore the domain we then explore the agreed-upon knowledge of different subsets of the expert group.

If all experts know about the same objects of the domain we can regard the group of experts as a triadic domain where each experts view is expressed as one condition.
In our running example, imagine that there are three experts for the bus and train schedule: One for Monday-Friday, one for Saturday and one for Sunday. The three experts will have different opinions about the implication theory of the time slots.

To explore the dependencies of attributes in this triadic multi-expert domain, we utilize triadic exploration.
To ask about a conditional implication then means to ask all experts if the implication holds in their view.
However, a simple translation back to the triadic case means that each time an expert gives a counterexample to a question, all experts must be consulted about their view on the counterexample to stay within the triadic setting (because we need the full slice of the triadic context). 
This is not ideal, but we can adapt the triadic exploration to avoid this issue:  Since we do not rely on any specific properties of the triadic context other than being able to form the subposition of the condition contexts $\K_d$, we can simply leave the triadic setting behind and transfer the idea of conditional attribute implications to a setting where we replace the triadic context with a context family on the same set of attributes (but not necessarily the same set of objects), i.e., a context family $\{\K_e=(G_e,M,I_e)| e\in E\}$ for a group of experts $E$.

Note that we could also explore the implication theory for each context in such a context family independently and combine the results afterwards, as initially suggested in \cref{sec:triadic-exploration}.
It is not obvious how this approach compares to the triadic one.
However, the triadic approach also allows to only explore a subset of the system of implication theories.
 
A real world example for such a context family can be found in the \emph{BSI-IT-Grundschutzkatalog}\footnote{ 
  \url{https://www.bsi.bund.de/EN/Topics/ITGrundschutz/itgrundschutz_node.html}}, a publication by the German \emph{Federal Office for Information Security}, which contains security recommendations on a wide variety of IT topics.  There, a general set of elementary threats is defined and for topics where these threats are present (for example organizational, infrastructure and personnel) a set of measures is defined where each measure combats one or multiple of the threats.  Hence, if we regard the elementary threats as attributes, the topics as conditions/experts and the measures as objects, we have a context family on the same set of attributes but with different object sets.

  To explore the domain of such a context family, where the set $G$ varies for the different conditions, we have to slightly alter \cref{alg:triadic-exploration,alg:explore-conditions}.  In particular, we need to replace the triadic contexts with context families and the conditions with the experts.  The triadic context of counterexamples becomes a context family of counterexamples where each expert has their own context of counterexamples and the objects between them can differ.  Hence, in \cref{alg:explore-conditions},  $A''$ is computed on the subposition of the respective contexts of counterexamples and asking about the implication $A\implies A''$ means asking each of the experts.  Now, counterexamples from one expert can be accepted without having to ask all other experts about their view on the example.
  
If we explore a triadic domain in this more abstract setting of context families on the same set of attributes, the trade-off is that we obtain less complete information about the counterexamples.  However, we still obtain the same knowledge in terms of conditional attribute implications that hold in the domain. 

In addition, we gain the ability to explore context families that do not fit the triadic setting or only do so after some modifications, as for example, the context family of the \emph{BSI-IT-Grundschutzkatalog}.  Another example can be derived from the running example:  If we look at $\K_{\text{Mo-Fr}}$ in \cref{fig:triadic-context-as-context-family}, imagine that instead of one context (and thus one expert) of bus and tram lines we had one for bus lines and one for tram lines.  Clearly, this family of two contexts could be transformed into a triadic context, however, to do so we would have to add bus lines to the tram lines context and vice versa -- mixing domains that might be perceived as different.

\section{Conclusion and Outlook}
\label{sec:conclusion-outlook}
In this paper, we addressed the problem of multi-expert attribute exploration in Formal Concept Analysis.  To this end, we developed \emph{triadic exploration} -- an analogue to attribute exploration -- for Triadic Concept Analyis, which extends Formal Concept Analysis with the notion of conditions.  Triadic exploration helps a triadic domain expert to explore the structure of the conditional attribute implications of the domain.

We adapted triadic exploration to a multi-expert setting by considering the experts' views of a domain as conditions in a triadic setting.  We discussed the ramifications of this approach and subsequently suggested to adapt triadic exploration to the more general setting of context families on the same set of attributes.

This paper is a step towards multi-expert exploration where experts can have different views on a domain.  In contrast to the few prior works on this subject, here the experts can have opposing views.  A next step is the combination of this approach with the notion of partial expert knowledge and a more in depth study of context families as a foundation for multi-expert explorations.

\printbibliography

\end{document}